\documentclass{article}

\usepackage{microtype}
\usepackage{graphicx}
\usepackage{subfigure}
\usepackage{booktabs} 
\usepackage{enumitem}

\usepackage{hyperref}
\usepackage{thm-restate}

\usepackage[accepted]{icml2025}

\usepackage{amsmath}
\usepackage{amssymb}
\usepackage{mathtools}
\usepackage{amsthm}

\usepackage[capitalize,noabbrev]{cleveref}

\theoremstyle{plain}
\newtheorem{theorem}{Theorem}[section]

\newtheorem{example}[theorem]{Example}
\newtheorem{assumption}[theorem]{Assumption}
\theoremstyle{definition}
\newtheorem{definition}[theorem]{Definition}
\theoremstyle{remark}

\usepackage[textsize=tiny]{todonotes}
\usepackage[T1]{fontenc}

\icmltitlerunning{Selecting Decision-Relevant Concepts in Reinforcement Learning}

\begin{document}
\newcommand{\nrcomment}[1]{{\color{red} Naveen: {#1}}}

\newcommand{\smcomment}[1]{{\color{blue} Steph: {#1}}}

\newcommand{\glcomment}[1]{{\color{violet} Grace: {#1}}}

\twocolumn[
\icmltitle{Selecting Decision-Relevant Concepts in Reinforcement Learning}


\begin{icmlauthorlist}
\icmlauthor{Naveen Raman}{yyy}
\icmlauthor{Stephanie Milani}{zzz} 
\icmlauthor{Fei Fang}{yyy}
\end{icmlauthorlist}

\icmlaffiliation{yyy}{Carnegie Mellon University}
\icmlaffiliation{zzz}{New York University}
\icmlcorrespondingauthor{Naveen Raman}{naveenr@cmu.edu}

\icmlkeywords{Concept-Based Models, Reinforcement learning, Interpretability, State Abstraction}

\vskip 0.3in
]

\printAffiliationsAndNotice{}

\begin{abstract}
    Training interpretable concept-based policies requires practitioners to manually select which human-understandable concepts an agent should reason with when making sequential decisions. 
This selection demands domain expertise, is time-consuming and costly, scales poorly with the number of candidates, and provides no performance guarantees.
To overcome this limitation, we propose the first algorithms for principled automatic concept selection in sequential decision-making. 
Our key insight is that concept selection can be viewed through the lens of state abstraction: intuitively, a concept is decision-relevant if removing it would cause the agent to confuse states that require different actions. 
As a result, agents should rely on \emph{decision-relevant} concepts; states with the same concept representation should share the same optimal action, which preserves the optimal decision structure of the original state space.
This perspective leads to the Decision-Relevant Selection (DRS) algorithm, which selects a subset of concepts from a candidate set, along with performance bounds relating the selected concepts to the performance of the resulting policy.
Empirically, DRS automatically recovers manually curated concept sets while matching or exceeding their performance, and improves the effectiveness of test-time concept interventions across reinforcement learning benchmarks and real-world healthcare environments.
\end{abstract}
\section{Introduction}
\label{sec:introduction}
Concept-based models produce interpretable agents by making predictions using human-understandable concepts~\cite{tcav}. 
These models consist of a concept predictor, which maps high-dimensional states (e.g., images) to concepts (e.g., ``Paddle Location > 0?''), and a policy which maps concept predictions to actions~\citep{concept_rl}.
Concept-based models have three advantages: 1) interpretability is built into the model, 2) poor decisions can be traced to concept prediction errors, and 3) humans can intervene on mispredicted concepts~\cite{intcem}. 

\begin{figure}[t]
    \centering 
    \includegraphics[width=0.4\textwidth]{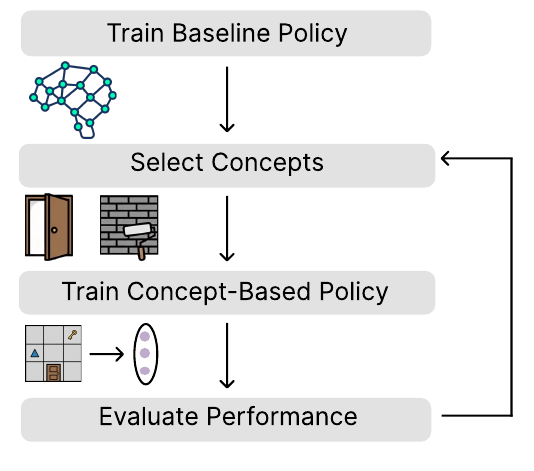}
    \caption{Standard pipeline for training concept-based policies. Practitioners select concepts for decision-making through a labor-intensive process of iteratively selecting candidate concepts, training concept-based policies, and evaluating their performance.}
    \label{fig:concept_pipeline}
\end{figure}

A key step in training these models is choosing the set of concepts for agent decision-making. 
As shown in Figure~\ref{fig:concept_pipeline}, practitioners distill a non-interpretable baseline policy into an interpretable concept-based policy by first selecting an initial set of concepts, then iteratively training and evaluating policies to refine the concept set.
This process is costly; it requires domain expertise, scales poorly with the number of concepts, and provides no performance guarantees.

In this work, we propose the first principled algorithms for concept selection in sequential decision-making (Figure~\ref{fig:pull_figure}). \footnote{To the best of our knowledge}
While prior work has explored using large language models (LLMs) to select concepts in supervised settings~\citep{labo,partially_shared_cbm}, there is little work investigating principled selection for sequential decision-making.  
Such a problem is more complex in sequential decision-making because states must be distinguished according to their long-term decision consequences rather than any immediate label.
For example, the ``Door Location'' concept is useful in Minigrid only if the agent can first obtain the key. 

\begin{figure*}[t]
    \centering 
    \includegraphics[width=\textwidth]{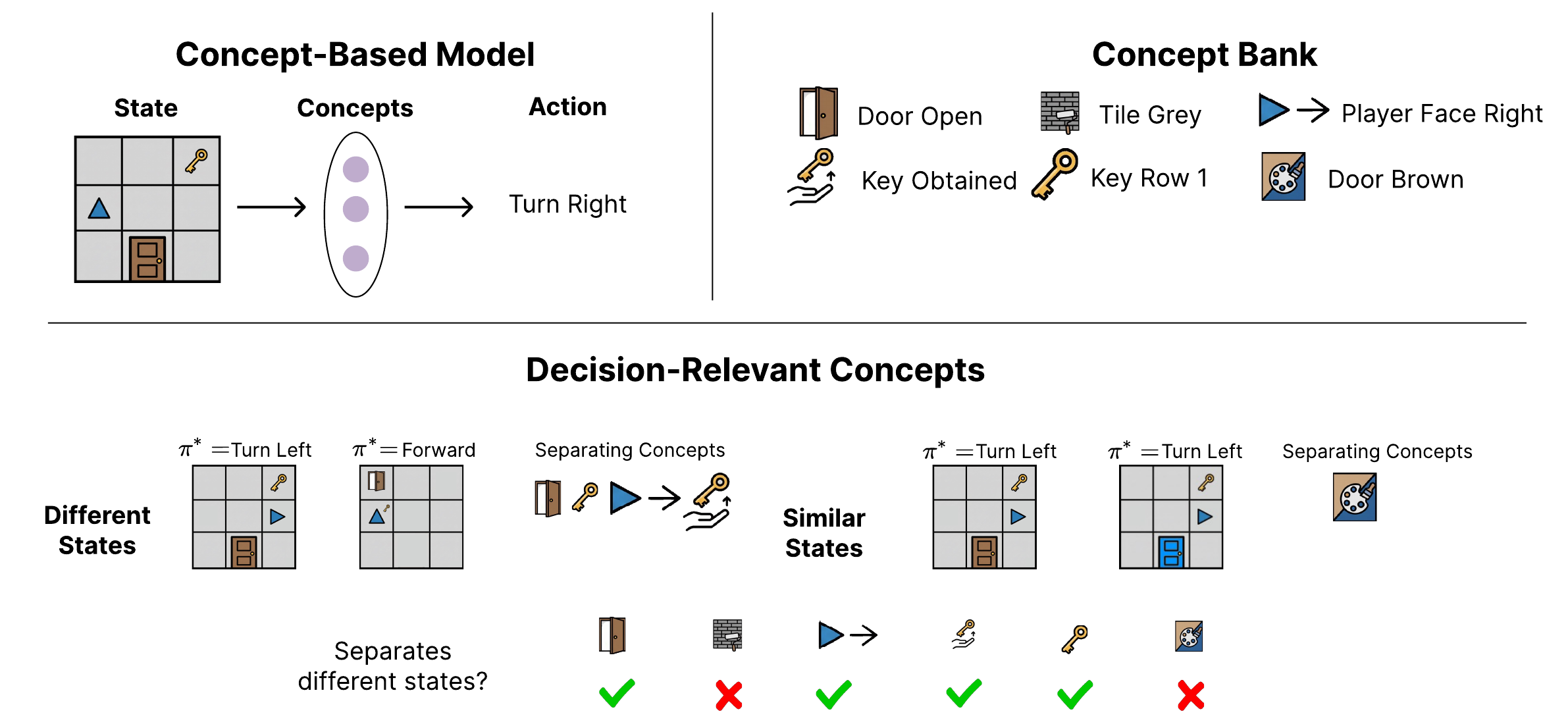}
    \caption{Concept-based models rely on a set of decision-relevant concepts that help distinguish between different states, yet currently these concepts are manually selected. In this work, we study how to identify and select decision-relevant concepts. Our key insight is that decision-relevant concepts best separate ``different'' states, where difference is defined by their decision consequences. We use this insight to develop algorithms for concept selection with performance guarantees.}
    \label{fig:pull_figure}
\end{figure*}

Automating this selection process offers three key benefits.
First, it eliminates manual concept engineering, substantially reducing the human effort needed to train these models.
Second, it improves interpretability by reducing the number of variables used for decision-making~\cite{simplicity_evaluation}.
Third, it maintains model performance on the tasks of interest. 
However, automatic concept selection is challenging because concept utility depends on long-term decision consequences, and naively evaluating all concept subsets requires training exponentially many policies.

To tackle this problem, we propose the decision-relevant selection (DRS) algorithm, which automatically selects concepts that best distinguish between states.
We formalize concepts as a case of \emph{state abstractions}. 
Using this perspective, DRS identifies the subset of concepts that minimizes state abstraction error, ensuring states with identical concept representations share the optimal action.

\textbf{Contributions:} We (i) design algorithms for selecting decision-relevant concepts in decision-making; (ii) derive performance bounds by connecting concept selection to state abstraction theory; (iii) demonstrate strong empirical performance across RL benchmarks and real-world healthcare environments; and (iv) show that well-selected concepts improve the effectiveness of test-time interventions. 
\footnote{Code is available at \url{https://github.com/naveenr414/concept_decisions}}
\section{Preliminaries}
\label{sec:preliminaries}

\paragraph{Reinforcement Learning} We study decision-making in RL for an agent in an infinite-horizon Markov decision process (MDP).
An MDP $(\mathcal{S},\mathcal{A}, P, R,\gamma)$ is a tuple consisting of a set of states $\mathcal{S}$, a set of actions $\mathcal{A}$, a transition function $P: \mathcal{S} \times \mathcal{A} \times \mathcal{S} \rightarrow [0,1]$ that gives the transition probabilities given state-action-state triplets, a reward function $R: \mathcal{S} \times \mathcal{A} \rightarrow \mathbb{R}$ that assigns a reward for each state-action pair, and a discount factor $\gamma \in [0, 1]$.
The agent learns a (possibly stochastic) policy $\pi: \mathcal{S}  \rightarrow \mathcal{A}$ that maps states to actions. 
We evaluate $\pi$ through the value $V^{\pi}(s)$ and action-value functions $Q^{\pi}(s,a)$: 
\[
V^\pi(s) := \mathbb{E}_\pi\!\left[\sum_{t=0}^{\infty} \gamma^t R(s_t,a_t) \mid s_0 = s \right]
\] 
\[
Q^\pi(s,a)\;:=\;\mathbb{E}_\pi\!\left[\sum_{t=0}^{\infty}\gamma^{t}\,R(s_t,a_t)\;\middle|\;s_0=s,\;a_0=a\right],
\] 
where $a_t = \pi(s_t),\; s_{t+1}\sim P(\cdot\mid s_t,a_t)$.
The aim is to find the optimal policy $\pi^{*}$, which maximizes $\mathop{\mathbb{E}}[V^{\pi^{*}}(s)].$

\paragraph{Concept-Based Models} Concept-based models ensure interpretability by mapping states to human-understandable concepts, then mapping concepts to actions (Figure~\ref{fig:pull_figure}). 
As in previous work~\cite{cbms}, each concept, $c_{i}$, is a Boolean-valued function, and the concept set is $\mathcal{C}= \{\,c_1, c_2, \dots, c_K\}, 
  c_i: \mathcal{S} \rightarrow \{0,1\}$. 
Binary concepts are standard~\citep{cem}, though our results can be extended to continuous concepts. 
Concept-based models consist of a concept predictor, $g_{\mathbf{c}}(s) := [c_{i_1}(s),\ldots,c_{i_k}(s)]$ and a policy $\pi_{\mathbf{c}}: \{0,1\}^{k} \rightarrow \mathcal{A}$, so that a state $s$ is mapped to an action $a$ as $\pi_{\mathbf{c}}(g_{\mathbf{c}}(s))$~\cite{concept_rl,concept_rl_2}.
Users can intervene on concept-based models during test-time by correcting a subset of concept predictions $g_{\mathbf{c}}(s).$

\paragraph{State Abstractions} 
We propose viewing concepts as a special case of \emph{state abstractions}. 
State abstractions group together states via a function $\phi_{\mathrm{SA}}: \mathcal{S} \to \mathcal{Z}$ such that states in the same abstraction satisfy certain desiderata.
Because concept representations may merge states that still differ in their true dynamics or value, we focus on \emph{$\epsilon$-approximate state abstractions}, which require grouped states to have near-identical action values under the optimal policy~\citep{approximate_state_abstraction}.
Formally, if $\phi_{\mathrm{SA}}(s_{1}) = \phi_{\mathrm{SA}}(s_{2})$, then $\max\limits_{a \in \mathcal{A}} \bigl|\,Q^{\pi^{*}}(s_1,a) - Q^{\pi^{*}}(s_2,a)\bigr|
    \;\le\; \epsilon.$
The state abstraction literature allows us to give guarantees based on the level of coverage for a given abstraction: 
\begin{theorem}[Theorem 1,~\cite{approximate_state_abstraction}]
\label{thm:approximate_state_abstraction}
    Suppose that $\phi_{\mathrm{SA}}$ is an $\epsilon$-approximate state abstraction. 
    Let $\pi_{\phi_{\mathrm{SA}}}$ be 
    \[
\pi_{\phi_{\mathrm{SA}}}
  = \arg\max_{\pi}
      \mathbb{E}_{s\sim\mathcal{S}}
      \bigl[ Q^{\pi^{*}}(\phi_{\mathrm{SA}}(s),\pi(\phi_{\mathrm{SA}}(s))) \bigr],
    \]
    Then $V^{\pi^{*}}(s) - V^{\pi_{\phi_{\mathrm{SA}}}}(s) \leq \frac{2 \epsilon}{(1-\gamma)^{2}}$ for all $s$. 
\end{theorem}
\section{Identifying Decision-Relevant Concepts}
\label{sec:concept_state_abstraction}
We now present the first formalism of the concept selection problem. 
We begin by defining the problem setting, then propose \emph{decision-relevancy} as the key principle for selecting concepts. 
To demonstrate the benefit of this principle, we establish theory showing that decision-relevancy can preserve task performance during training and yield greater performance increases during test-time interventions.

\subsection{Problem Setting}
\label{sec:problem_setting}
Concept predictors $g_{\mathbf{c}}$ rely on the selection of a set of $k$ concepts $\mathbf{c}$ from a larger bank of $K$ concepts $\mathcal{C}$.
Here, $k$ controls the number of concepts selected; a larger $k$ can hinder interpretability~\cite{simplicity_evaluation} by introducing additional complexity. 
Concept banks are obtained either manually or through large language models~\citep{label_free_cbms}.
While any $\mathbf{c}$ is semantically meaningful, the reward can vary: 

\begin{example}
Consider a state space $\mathcal{S} = \{1,2,3,4\}$ and an action space $\mathcal{A} = \{\mathrm{left},\mathrm{stay},\mathrm{right}\}$, which corresponds to an agent navigating a 1D state space with 4 states that loop around. 
Let $R(1,\cdot) = R(3,\cdot) = 1$ and $R(2,\cdot) = R(4,\cdot) = 0$. 
Consider two binary concepts: \[
\begin{aligned}
c_1(s) &= \mathbf{1}\{s\bmod 2 = 0\},\\
c_2(s) &= \mathbf{1}\{s\bmod 3 = 0\},\\
\end{aligned}
\]
where $\mathbf{1}$ is the indicator function. 
When $k=1$, selecting $\mathbf{c} = \{c_{1}\}$ ensures that states sharing the same parity have identical rewards and symmetric transitions under $\pi^{*}$. 
However, selecting $\mathbf{c} = \{c_{2}\}$ means the reward is not consistent; while $g_{\mathbf{c}}(1) = g_{\mathbf{c}}(2)$, $R(1,\mathrm{stay}) \neq R(2,\mathrm{stay}).$
A policy trained on $c_{1}$ will achieve a reward of 9.5 while a policy trained on $c_{2}$ will achieve a reward of 9.1 because it will take a suboptimal action for $s=1$. 
\end{example}

\paragraph{Objective}
Concept selection algorithms select $\mathbf{c} \subseteq \mathcal{C}$, where $\mathcal{C}$ is the universe of concepts, and can come from human curation, LLM generation, or other sources. 
The goal is to maximize the performance of the policy $\pi^{*}_{\mathbf{c}}$: 
\begin{equation}
\max_{\mathbf c:|\mathbf c|\le k}
\mathbb{E}_{s \sim \mu}
\bigl[
Q^{\pi^{*}_{\mathbf c}}(s, \pi^{*}_{\mathbf c}(g_{\mathbf c}(s)))
\bigr].
\label{eq:objective}
\end{equation}
Here, $\mu \in \Delta(\mathcal{S})$ is the initial state distribution, and $\pi^{*}_{\mathbf{c}}$ is the optimal policy given the concept predictor $g_{\mathbf{c}}$. 
Unlike classical state abstraction, which learns abstract state representations without constraints, we restrict our attention to selecting a subset of concepts from a fixed, human-interpretable concept bank.
Our problem requires combinatorial optimization over a fixed concept bank, making it NP-hard to compute the optimal solution (proof in Appendix~\ref{sec:proofs}). 

\subsection{Defining Decision-Relevancy}
\label{sec:identifying_relevant_concepts}
In Equation~\ref{eq:objective}, the set of concepts $\mathbf{c}$ help distinguish between states through the concept predictor $g_{\mathbf{c}}.$ 
In other words, concepts need to separate states with different optimal actions. 
As a result, we propose identifying \emph{decision-relevant} concepts as the key principle for concept selection. 

Intuitively, concepts are decision-relevant only to the extent that they distinguish
states in which different actions lead to different outcomes.
A natural condition is that if $\pi^{*}(s) \neq \pi^{*}(s')$,
then $g_{\mathbf{c}}(s) \neq g_{\mathbf{c}}(s')$, where $s$ and $s'$ are pairs of states.
A stronger desideratum is that if
$Q^{\pi^{*}}(s,\cdot) \neq Q^{\pi^{*}}(s',\cdot)$,
then $g_{\mathbf{c}}(s) \neq g_{\mathbf{c}}(s')$.
The former helps us better learn policies, while the latter is needed to guarantee good performance from the resulting policy.
While our main concern is the latter condition, the former is useful for guiding algorithm development. 
To formally define decision-relevance, we first introduce two ideas: 
\begin{definition}[Q-Distance]
For two states $s, s' \in \mathcal{S}$, the Q-distance is
$D_{s,s'} \coloneqq \max\limits_{a \in \mathcal{A}}
\lvert Q^{\pi}(s,a) - Q^{\pi}(s',a) \rvert$.
\end{definition}
\begin{definition}[Abstraction Error]
    Define the abstraction error for a concept predictor $g$ as 
    \[
\epsilon_{Q^\pi}(g) := \max_{s,s' : g(s)=g(s')}  D_{s,s'}.
\]
Unless otherwise stated, let $\epsilon(g) = \epsilon_{Q^{\pi}}(g).$
\end{definition}


\subsection{Benefits of Decision-Relevant Concepts}

\paragraph{Improved Performance}
We first leverage results from the state abstraction literature~\citep{approximate_state_abstraction} to provide guarantees for a set of concepts $\mathbf{c}$ based on $\epsilon_{Q^{\pi}}$: 
\begin{restatable}{proposition2}{propconceptabstraction}
\label{thm:approximate_concept_abstraction}
For a concept predictor $g_{\mathbf{c}}$, we have that 
\[
V^{\pi^{*}}(s) - V^{\pi^{*}_{\mathbf{c}}}(s) \leq \frac{2 \epsilon(g_{\mathbf{c}})}{(1-\gamma)^2}.
\]
\end{restatable}

We prove that $g_{\mathbf{c}}$ induces an $\epsilon$-state abstraction, allowing us to use Theorem~\ref{thm:approximate_state_abstraction} (full proof in Appendix~\ref{sec:proofs}). 
Intuitively, $\epsilon$ captures the quality of the abstraction; if $\epsilon$ is low, then the abstraction is fine-grained, leading to better performance because $\pi^{*}_{\mathbf{c}}$ can better distinguish between states. 
Obtaining a tighter bound is difficult because the $1/{(1-\gamma)^2}$ term is standard in approximate state abstraction~\cite{approximate_state_abstraction}.

\paragraph{Receptivity to Test-Time Intervention}
During test-time intervention, a (potentially simulated) user intervenes on an $\alpha$-fraction of concepts, raising their accuracy from $\delta^{A}$ to $\delta^{H}$. 
Here, the exact subset of intervened concepts is unknown until test-time. 
Intuitively, when interventions improve concept accuracy, policies using decision-relevant concepts benefit more from human effort:

\begin{assumption}
\label{asm:monotone_dominance}
For any two concept sets $\mathbf c, \mathbf c'$ and any two accuracy vectors
$\delta' \preceq \delta$, if
\[
\mathbb{E}[V^{\pi^{*}_{\mathbf c,\delta}}(s)]
\ge
\mathbb{E}[V^{\pi^{*}_{\mathbf c',\delta}}(s)],
\]
then
\[
\mathbb{E}[V^{\pi^{*}_{\mathbf c,\delta'}}(s)]
\ge
\mathbb{E}[V^{\pi^{*}_{\mathbf c',\delta'}}(s)].
\]
\end{assumption}

\begin{assumption}
    \label{asm:abstraction_dominance}
    For any two concept sets $\mathbf{c}, \mathbf{c'}$, when $\delta=\mathbf{1}$, if $\epsilon(g_{\mathbf{c}}) \leq \epsilon(g_{\mathbf{c}'})$, then  
    \[
    \mathbb{E}[V^{\pi^{*}_{\mathbf c,\delta'}}(s)]
    \ge
    \mathbb{E}[V^{\pi^{*}_{\mathbf c',\delta'}}(s)].
    \]
\end{assumption}

Assumption~\ref{asm:monotone_dominance} holds when concept errors are independent across concepts, or when $\mathbb{E}[V^{\pi^{*}_{\mathbf{c},\delta}}] = f_{\pi^{*}}(\mathbf{c}) + g_{\pi^{*}}(\delta)$ and $g$ is increasing in $\delta$. 
Assumption~\ref{asm:abstraction_dominance} holds if Proposition~\ref{thm:approximate_concept_abstraction} is tight, or more generally if $V^{\pi^{*}}(s)-V^{\pi^{*}_{\mathbf{C}}}(s) \leq \frac{2 \epsilon}{(1-\gamma)^2}$ is constant across values of $\mathbf{c}.$

\begin{restatable}{theorem2}{theoremhuman}
Let $\mathbf c$ and $\mathbf c'$ be two concept sets, and let
$\epsilon(g_{\mathbf c}) \le \epsilon(g_{\mathbf c'})$.
Under Assumptions~\ref{asm:monotone_dominance} and~\ref{asm:abstraction_dominance}, for any intervention level $\alpha \in [0,1]$ and any accuracy vector
$\delta^{(\alpha)}$ obtained by improving an $\alpha$-fraction of concepts from
$\delta^{A}$ toward $\delta^{H}=\mathbf 1$, we have
\[
\mathbb{E}[V^{\pi^{*}_{\mathbf c,\delta^{(\alpha)}}}(s)]
\ge
\mathbb{E}[V^{\pi^{*}_{\mathbf c',\delta^{(\alpha)}}}(s)].
\]
\end{restatable}
Intuitively, when interventions improve concept accuracy, policies built on more decision-relevant concepts benefit more from the same level of human effort.
We provide empirical support for these Assumptions in Section~\ref{sec:experiments} and these assumptions tend to hold because well-selected concepts lead to better performance across accuracy values. 
\section{Concept Selection Algorithms}
\label{sec:imperfect}

We build on our problem formulation and optimality notion to propose two concept selection algorithms. 
The algorithms tackle perfect and imperfect concept predictors respectively. 

\subsection{Selection with Perfect Concept Predictors}
We introduce an algorithm for concept selection based on the intuition from Section~\ref{sec:identifying_relevant_concepts} called \textbf{Decision-Relevant Selection (DRS)}. 
The DRS algorithm chooses concepts $\mathbf{c}$ that minimize $\epsilon(g_{\mathbf{c}})$ while separating states with different actions (via a hyperparameter $\rho$).  
The former is necessary to guarantee good performance, while the latter allows us to more easily learn good policies. 
Formally, let $\mathbf{x}$ be a 0-1 vector representing the subset of concepts selected, and $\mathbf{Y}$ a binary matrix representing pairs of states differing in concept representations. 
We aim to minimize the abstraction error $D_{s,s'}$ across state pairs $s,s'$ with $Y_{s,s'} = 1$, while limiting concept selection to $k$ concepts:

\begin{subequations}
\label{lp:concept_selection}
\begin{align}
\min_{\mathbf{x}, \mathbf{Y}} \quad & \sum_{s,s' \in \mathcal{S}} D_{s,s'} (1 - Y_{s,s'}) \label{eq:P1} \\[4pt]
\text{s.t.} \quad 
    & \sum_{j=1}^{K} x_j \le k \label{eq:P1a} \\
    & Y_{s,s'} \le \sum_{j=1}^{K} x_j \, \mathbf{1}[c_j(s) \neq c_j(s')],
        \quad \forall\, s,s' \in \mathcal{S} \label{eq:P1b} \\
    & Y_{s,s'} \le Y_{t,t'} \notag \\
    & \qquad \forall\, (s,s'),(t,t') \in \mathcal{S}^2 
        \text{ s.t. } D_{s,s'} \le D_{t,t'} \label{eq:P1c} \\
    & \sum_{s,s' \in \mathcal{S}} \mathbf{1}[\pi(s) \neq \pi(s')] \, Y_{s,s'} \notag \\
    & \qquad \ge \rho \sum_{s,s' \in \mathcal{S}} \mathbf{1}[\pi(s) \neq \pi(s')] \label{eq:P1d} \\
    & x_j \in \{0,1\}, \quad Y_{s,s'} \in \{0,1\}. \label{eq:P1e}
\end{align}
\end{subequations}

\paragraph{Practical Considerations} 
Because the state space can be large, we sample from agent rollouts and consider only abstract states; do not differentiate states with $c_{j}(s) = c_{j}(s') \forall j \in [K]$.  
Let $n_d$ denote the number of \emph{distinct abstract states} observed, i.e., distinct values of $g_{\mathbf c}(s)$.
While $k$ binary concepts can induce up to $2^{k}$ abstract states, rollouts concentrate on a small subset due to environmental constraints and policy-induced structure.
Two structural properties keep $n_d$  small in practice: (1) environmental constraints limit which concept combinations are physically realizable, and (2) policy-induced structure concentrates rollouts on a small subset of reachable abstract states.
States sharing the same abstract representation can be aggregated, yielding an MILP with $O(n_d^2 + K)$ variables/constraints (runtimes and further discussion in Appendix~\ref{sec:runtime}). 

\subsection{Performance Guarantees}
\label{sec:performance_guarantees}
We provide performance guarantees for the DRS algorithm. 
\begin{restatable}{theorem2}{theoremperfect}
Let $(\mathbf x^*, \mathbf y^*)$ be an optimal solution to Equation~\ref{eq:P1} with $\rho=0$, and let
$\mathbf c = \{c_j : x_j^* = 1\}$.
$g_{\mathbf c}$ minimizes $\epsilon(g_{\mathbf{c}})$ among predictors using $k$ concepts:
\[
\epsilon(g_{\mathbf c})
\;=\;
\min_{g : |\mathbf c_g| \le k} \epsilon(g).
\]
\end{restatable}

The solution exactly captures the optimization problem over $\epsilon(g_{\mathbf{c}})$, so the resulting concepts have small abstraction error. 

While $Q^{\pi}(s,a)$ might only be approximately known, the DRS algorithm is robust to errors in $Q^{\pi}$: 
\begin{restatable}{theorem2}{theoremrobust}
\label{thm:robust}
Assume that $\max\limits_{s,a} |\hat Q^{\pi}(s,a) - Q^{\pi}(s,a)| \le \delta$.
Let $(\mathbf{x}^*, \mathbf{Y}^*)$ be an optimal solution to Equation~\ref{eq:P1} constructed
using $D_{s,s'}$, and let $\mathbf c^* = \{c_j : x_j^* = 1\}$.
Let $\hat D_{s,s'} = \max_a |\hat Q^{\pi}(s,a)-\hat Q^{\pi}(s',a)|$, and let
$(\hat{\mathbf x}, \hat{\mathbf Y})$ be an optimal solution to Equation~\ref{eq:P1} constructed
using $\hat D_{s,s'}$, with induced abstraction $g_{\hat{\mathbf c}}$.
Then
\[
\epsilon_{Q^{\pi}}(g_{\hat{\mathbf c}}) \le \epsilon_{Q^{\pi}}(g_{\mathbf c^*}) + 4\delta,
\]
and for all $s \in \mathcal S$,
\[
V^{\pi_{\mathbf{c}^{*}}}(s) - V^{\pi_{\hat{\mathbf{c}}}}(s)
\le \frac{2(\epsilon_{Q^{\pi}}(g_{\mathbf c^*})+ 4\delta)}{(1-\gamma)^2}.
\]
\end{restatable}

We prove this in Appendix~\ref{sec:proofs} by bounding $|\hat{D}_{s,s'}-D_{s,s'}|$. 

\subsection{Selection with Imperfect Concept Predictors}
When concept predictors are imperfect, state separation (Equation~\ref{eq:P1b}) is stochastic, requiring new algorithms for concept selection. 
To tackle this, we (i) introduce a noise model for concept predictors to capture imperfections, (ii) show that no method can guarantee low abstraction error under adversarial perturbations, and (iii) derive a lower bound on state separation under stochastic perturbations.

\paragraph{Noisy Concept Predictors} 
Let $f^{T}_{\mathbf{c}} = [f^{T_1}_{c_1},\ldots,f^{T_{k}}_{c_{k}}]$ be a vector of perturbation functions so $T_{i} \subseteq \mathcal{S}$ and $\frac{|T_{i}|}{|\mathcal{S}|} = 1-\delta_{i}$. 
Here, $T_{i}$ are the states where we flip concept predictions: 
\[
f^{T_i}_{c_i}(s)
:=
\begin{cases}
1 - c_i(s), & \text{if } s \in T_i, \\[4pt]
c_i(s), & \text{if } s \notin T_i .
\end{cases}
\]

\paragraph{Limits Under Adversarial Noise} We note that while we can extend concept selection to imperfect scenarios, any algorithm will perform poorly under adversarial noise.  
Even if all concept predictors are near-perfect, there always exist perturbations so that the induced policy performs poorly:
\begin{restatable}{lemma2}{lemmaperturbation}
Suppose that for each concept $i$, an adversary may apply a perturbation
$f^{(T_i)}$ that flips the value of $c_i$ on a subset $T_i \subseteq \mathcal S$
with $|T_i|/|\mathcal S| = 1-\delta_i$. 
Then for any concept predictor $g_{\mathbf c}$ with $\delta_{i} < 1$, there exists a collection of
perturbations $
f = \{f^{T_1}_{c_{1}},\ldots,f^{T_k}_{c_{k}}\}$
such that
$\epsilon(f(g_{\mathbf c})) = \max_{s,s'} D_{s,s'}.$
\end{restatable}
We prove this by constructing adversarial perturbation functions $f$ that fail to distinguish between states $s$ and $s'$ with maximum values of $D_{s,s'}$, showing that worst-case robustness is fundamentally impossible (proof in Appendix~\ref{sec:proofs}).

\paragraph{State Separation Under Noisy Concept Predictors} 
In the DRS algorithm, we enforce state separation through $Y_{s,s'} \leq \sum_{j=1}^{K} x_{j} \mathbf{1}[c_{j}(s) \neq c_{j}(s')]$. 
When concept predictors are noisy, this constraint holds stochastically. 
To introduce this, let $\delta = [\delta_{1},\ldots,\delta_{k}]$ be a vector representing the accuracy of each concept, where each concept predictor has stochastic and independent errors. 
Next, note that the probability that noisy predictions
remain distinct is 
$\mathbb{P}[g_{\mathbf{c}}(s)_{j}\neq g_{\mathbf{c}}(s')_{j} \mid c_j(s)\neq c_j(s')]
= \delta_j^2 + (1-\delta_j)^2$,
since disagreement is preserved when both predictors are either correct
or incorrect.
For concepts with $c_j(s)=c_j(s')$, we conservatively ignore
noise-induced separation.
Because the product form is non-convex, we apply a logarithmic transformation to obtain a concave lower bound: 
\[
\log\!\left(1-Y_{s,s'}\right)
  \ge
\sum_{j : \mathbf{1}[c_j(s)\neq c_j(s')]}^{K}
x_{j} 
\log\!\bigl(1-(\delta_j^2+(1-\delta_j)^2)\bigr).
\]

We construct a new algorithm, called \textbf{DRS-log}, which uses a log-based constraint for state separation.
In practice, we optimize this efficiently using standard convex nonlinear primitives supported by modern optimization solvers. 
\begin{figure*}[ht!]
    \centering 
    \includegraphics[width=\textwidth]{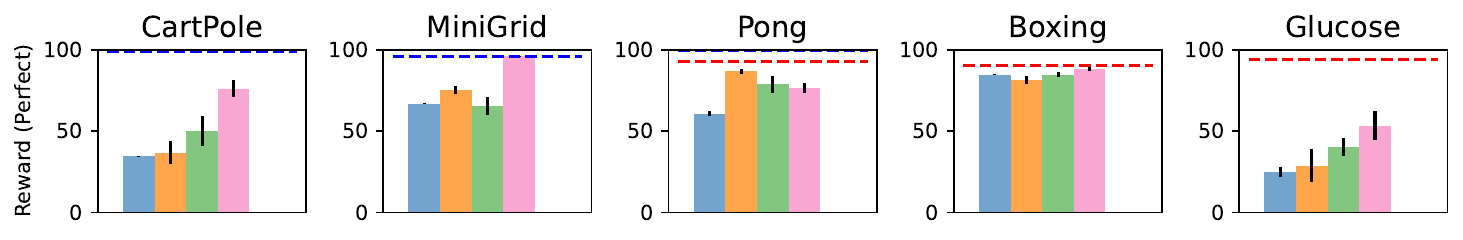}
    \includegraphics[width=\textwidth]{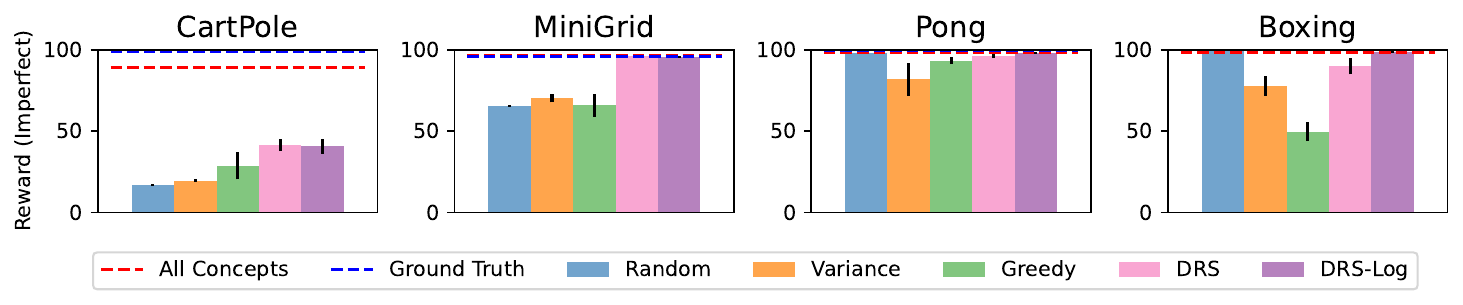}
    \caption{Normalized reward of concept selection algorithms with perfect (top) and imperfect (bottom) concept predictors. Our algorithm, DRS, improves performance compared to the random, variance, and greedy baselines for four out of five environments in the perfect setting. 
    DRS and DRS-log improve performance or are optimal in all environments in the imperfect setting. }
    \label{fig:comparison}
\end{figure*}

\section{Experiments}
\label{sec:experiments}
We evaluate our algorithms by focusing on three questions: 
\begin{enumerate}[nosep]
    \item \textbf{RQ1:} What is the impact of decision-relevant concepts on performance?  
    \item \textbf{RQ2:} What is the impact of decision-relevant concepts on test-time intervention? 
    \item \textbf{RQ3:} How do automatically selected concepts compare with manually selected ones? 
\end{enumerate}

\subsection{Experimental Setup}
\label{sec:experimental_setup}
\paragraph{Environments} We evaluate on CartPole, MiniGrid, Pong, Boxing, and Glucose~\cite{rl_textbook,mini_grid,atari,glucose}. 
Concepts are discretizations of human-understandable features. 
For example, in CartPole, we discretize ``Paddle Y'' and ``Ball Velocity X'' into ``Ball Velocity X < -1'' and ``Paddle Y > 0'' (see Appendix~\ref{sec:environments_detail}). 

\paragraph{Training Details} We evaluate with perfect and imperfect concept predictors, where imperfect concept predictors use CNNs to predict concepts. 
In Appendix~\ref{sec:environments_detail}, we provide exact accuracy values for each concept predictor, but note that all concept predictors achieve at least 95\% accuracy except for CartPole (which achieves 85\% due to velocity-related concepts). 
Because concept predictors achieve 100\% accuracy for Glucose, we report only the perfect setting for that environment.
We run all experiments with 6 seeds. 

\begin{figure}[t]
    \centering 
    \includegraphics[width=0.5\textwidth]{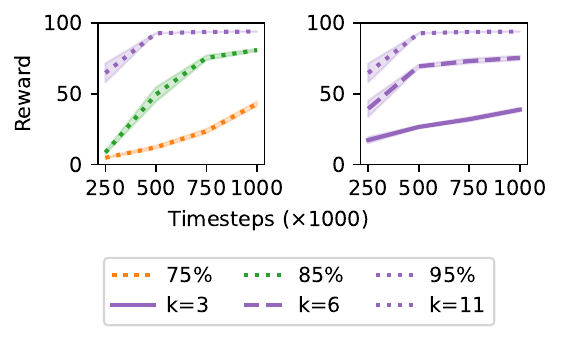}

    \caption{We vary the number of timesteps that we train policies for in MiniGrid, while also varying the accuracy of concept predictors (left) or the number of concepts selected (right). Increasing the accuracy of concept predictors speeds up training, while increasing the number of concepts increases the maximum performance.}
    \label{fig:num_concepts}
\end{figure}

\begin{figure*}[]
    \centering 
    \includegraphics[width=\textwidth]{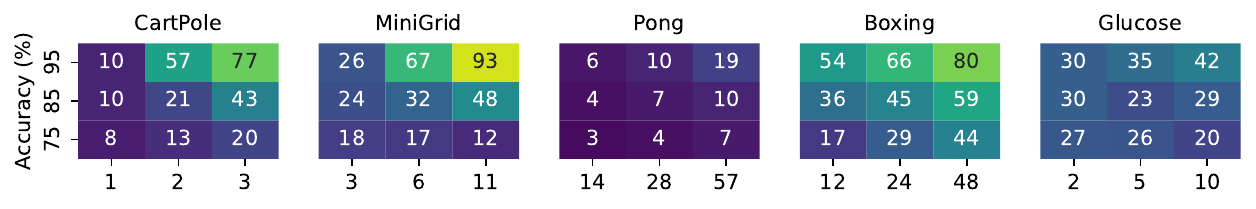}

    \caption{Impact of the number of concepts selected against the accuracy of the underlying concepts. 
    Increasing either number or accuracy of concepts has a similar impact on performance, and that sufficiently accurate and many concepts are needed to ensure good performance.}
    \label{fig:num_concepts_accuracy}
\end{figure*}

\begin{figure*}[h]
    \centering 
    \includegraphics[width=\textwidth]{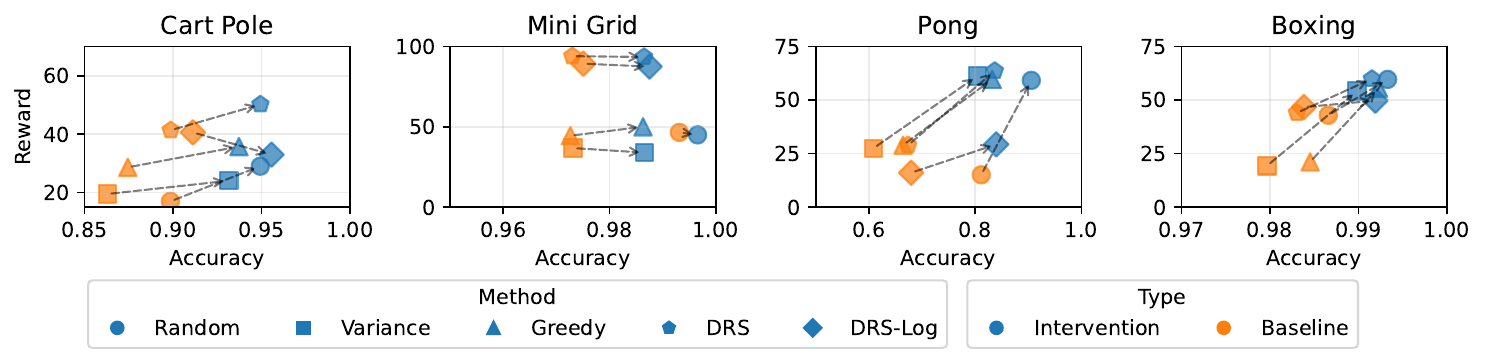}
    \caption{We assess the impact of concept selection upon test-time intervention across four environments. Across all four environments, the DRS algorithm has the highest reward both before and after intervention, showing that well-selected concepts improve intervention. }
    \label{fig:intervention_comparison}
\end{figure*}

\paragraph{Algorithms} 
We train concept predictors using NatureCNN~\cite{nature_cnn} and policies using proximal policy optimization (PPO)~\cite{ppo} (details in Appendix~\ref{sec:environments_detail} and Appendix~\ref{sec:algorithm_details}).
For the DRS and DRS-log algorithms we fix $\rho=0.75$ and evaluate other choices in Appendix~\ref{sec:other_algorithms}. 
In DRS, we relax the constraint corresponding to Equation~\ref{eq:P1c} because it is overly pessimistic by enforcing concept selection strictly in the order determined by state abstraction error. 
We compare against three baselines: 1) Random, which selects a random $k$-subset of concepts; 2) Variance, which selects concepts by ranking them according to their variation between being active ($c_{i}(s) = 1$) and inactive ($c_{i}(s) = 0$); 3) Greedy, which selects concepts that best explain downstream behavior by conditioning on actions and iteratively choosing those that most reduce variability in Q-values.
We normalize all rewards on a 0-100 scale, where the extremes are the min and max per-environment reward.  

\subsection{RQ1: Impact on Performance}
\label{sec:selection_experiments}
\paragraph{Decision-relevant concepts improve performance} We compare concept selection algorithms in both perfect (top) and imperfect (bottom) settings in Figure~\ref{fig:comparison}.
For each environment, we fix $k=\frac{K}{4}$.\footnote{
In Appendix~\ref{sec:varied_k}, we vary $k$ and find similar results.}
For perfect settings, DRS performs best for four out of five environments, performing 159\% better than the best baseline in CartPole, 28\% better in MiniGrid, 5\% better in Boxing, and 44\% better in Glucose. 
While DRS performs poorly in perfect settings with Pong, it makes up for this in the imperfect setting, as the selected concepts are better able to leverage uncertainty. 
In imperfect settings, DRS and DRS-log either outperform baselines or are near-optimal across all four environments. 
In the MiniGrid, Pong, and Boxing environments, both DRS and DRS-log are optimal, while in CartPole, both DRS and DRS-log perform at least 58\% better than alternatives. 
In Appendix~\ref{sec:qualitative}, we give qualitative insights.

\paragraph{Decision-relevant concepts increase maximum performance, while concept predictor accuracy increases training efficiency}
To better understand how well-selected concepts impact performance, in Figure~\ref{fig:num_concepts}, we plot training curves for MiniGrid policies in two settings. 
On the left, we vary concept predictor accuracies, and on the right, the number of concepts selected. 
Increases in accuracy improve training efficiency; fewer timesteps are needed to achieve the same performance. 
Increases in the number of concepts increase the maximum performance after full training, and we demonstrate results with CartPole in Appendix~\ref{sec:num_concepts_cartpole}. 

In Figure~\ref{fig:num_concepts_accuracy}, we show that accurate concept predictors (y axis) and many concepts (x axis) are needed to achieve near-optimal reward.  
For example, in CartPole, using only $k=2$ concepts caps performance at 38, while having 85\% accurate concept predictors caps performance at 43, compared to 77 with 95\% accurate predictors at $k=3$. 
While well-selected concepts are necessary to ensure effective concept-based models, they are not sufficient. 
The exact number of concepts and accuracy level is environment-dependent; in Glucose and Pong even 95\% accurate concept predictors are not sufficient, while in MiniGrid, 95\% accurate concept predictors achieve a normalized reward of 93.

\subsection{RQ2: Impact on Test-Time Intervention}
\label{sec:intervention_experiments}

\begin{figure}
    \centering 
    \includegraphics[width=0.45\textwidth]{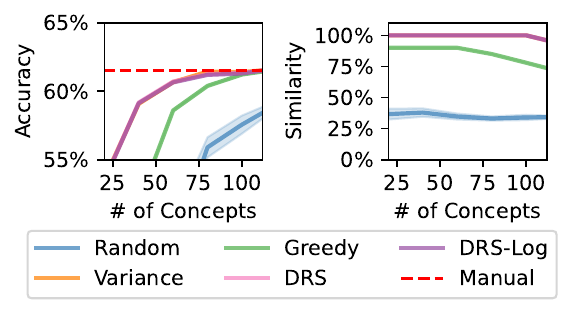}
    \caption{The DRS algorithm can mimic manual concept selection in CUB. It matches performance (left)  while selecting a subset of the manually selected concepts (right). }
    
    \label{fig:cub}
\end{figure}

\paragraph{Decision-relevant concepts improve test-time intervention}
We evaluate the impact of concept selection on test-time intervention in the imperfect setting.
We weaken policies and concept predictors by training for fewer timesteps, which better mimics scenarios that require intervention. 
In Figure~\ref{fig:intervention_comparison}, we compare reward and accuracy with and without intervention, while fixing $\alpha=0.5$ (we vary $\alpha$ and find similar results in Appendix~\ref{sec:intervention_percentage}) and selecting a random subset of concepts to intervene upon. 
Across environments, DRS maximizes reward after intervention, performing 40\% better in CartPole, 87\% better in MiniGrid, 4\% better in Pong, and 0.5\% better in Boxing than all alternatives. 
Well-selected concepts enhance the efficacy of test-time intervention because intervening on a small set of critical concepts leads to large performance gains.
We note that DRS-log can perform worse under intervention because it is optimized for the original predictor accuracy profile. 
When intervention renders a subset of concept predictors perfectly accurate, the originally selected subset may no longer be appropriate.

\subsection{RQ3: Automatic and Manually Selected Concepts}
\label{sec:similarity}

\paragraph{Decision-relevant concepts match performance of manual selection}
We use the Caltech-UC Davis Birds (CUB) dataset~\cite{cub}, a standard supervised task in concept-based learning. 
Most prior work uses a manually selected subset of 112 of the original 312 concepts, so we investigate how selection algorithms compare with manual selection. 
We start from the original 312 concepts and compare concept selection algorithms. 
We detail how we extend algorithms to the supervised setting in Appendix~\ref{sec:supervised_algorithms}. 

In Figure~\ref{fig:cub}, we show that automatically selecting concepts can replicate manual selection. 
We find that the variance, DRS, and DRS-log algorithms select the same concepts as manual selection because the concepts that best predict labels are the ones with the highest variance, resulting in the same performance when $k=112$. 
These algorithms can replicate manual selection using $k=80$ concepts while achieving within 0.6\% of manual performance. 
\section{Related Work}
\label{sec:related}
\paragraph{Interpretable Reinforcement Learning}
Work in interpretability examines how we can explain model predictions through methods that are trustworthy, informative, and transferable~\cite{model_mythos}. 
In an RL setting, prior work investigates feature importance-based methods~\cite{feature_based} and policy-level explanations~\cite{policy_level}.
Work in the concept-based RL space builds upon policy-level explanations to construct concept policy models~\cite{state2explanation} and extends these to multi-agent~\cite{concept_rl} and scarce-label settings~\cite{licorice_concepts}. 
Our work builds upon these policy-level explanations to better understand concept-based policy construction. 

\paragraph{Concept-Based Interpretability}
Concept-based explanations fall into two categories: post-hoc methods~\cite{tcav} and interpretable models~\cite{cbms,cem}. 
Key to the success of both methods is the selection of concepts. 
While this parallels the classic feature selection problem~\cite{feature_selection}, our work differentiates itself by selecting human-interpretable concepts based on their decision relevance rather than predictive performance. 
In the supervised setting, concepts are often selected manually~\cite{cbms,cem}, while in the unsupervised setting, concepts are selected via LLMs~\cite{label_free_cbms,labo}.
We differentiate ourselves from~\citet{concept_completeness}, which analyzes alignment between concept vectors and states in supervised learning, because we 1) formalize the concept selection problem, 2) focus on interpretable models rather than post-hoc, and 3) focus on decision-relevance in RL. 

\paragraph{State Abstractions}
We leverage state abstraction theory to generate performance guarantees for concept selection. 
Work in state abstraction imposes various desiderata on states within an abstraction group, including \emph{exact bisimulation}~\citep{equivalence_notion} or \emph{$\varepsilon$–optimality} criteria~\citep{unified_state_abstraction}. 
At a high level, an exact bisimulation abstraction groups states that behave identically under all possible actions: that is, they yield the same immediate rewards and have the same transition patterns to abstract successor states. 
Beyond exact bisimulation, other work studies approximate abstractions~\cite{approximate_state_abstraction} and abstraction selection~\cite{abstraction_selection}. 
While both our work and state abstractions aim to aggregate states according to various desiderata, our goal is to ensure interpretability via state selection, while state abstraction focuses on improving sample complexity for non-interpretable models. 
\section{Discussion and Conclusion}
Our work investigates how to reduce manual concept engineering by selecting concepts in an automated and principled manner. 
Our key intuition is that decision-relevant concepts should distinguish between states that are different.
We use results from state abstractions to quantify what ``different'' means, and use this to bound performance for a selected subset of concepts. 
In practice, we show that our proposed algorithms can identify concepts most relevant for decision-making, which improves performance both in standalone settings and in human–AI collaboration. 
Our work takes a step toward understanding how explanation structure shapes decision quality and human–AI collaboration.
Future work could extend our algorithms to non-binary concepts and more complex settings.

\newpage 

\bibliography{references}
\bibliographystyle{icml2025}

\newpage 

\appendix
\section*{Impact Statement}
This paper presents work whose goal is to advance the field of machine learning. There are many potential societal consequences of our work, none of which we feel must be specifically highlighted here. 

\section*{Acknowledgements}
We thank Mateo Espinosa-Zarlenga, Grace Liu, Nicholay Topin, and Santiago Cortes-Gomez for their comments on an earlier version of this work, and Sarah Cen for a discussion on concept selection. 
Co-author Raman is supported by an NSF GRFP award. 
This work was supported in part by NSF grant IIS-2046640 (CAREER).

\section{Environment Details}
\label{sec:environments_detail}
We provide details on the construction of glucose environment,  concept banks, and the training of concept predictors and groundtruth models in this section. 

\paragraph{Glucose Environment Construction}
We use the glucose environment~\cite{glucose} as a test-bed to understand the performance of our algorithm in real-world healthcare environments. 
Our observation space is a low-dimensional vector consisting of six normalized features derived from the current blood glucose (BG), its short-term trend, carbohydrate intake, insulin-on-board, and recent dosing history. 
The observation space is defined as a bounded continuous box, while the action space is discrete and corresponds to a set of clinically realistic bolus doses.
Each step corresponds to $5$ minutes of simulated time.  
At every step, a discrete action is mapped to an insulin dose (in units), which is then passed to the underlying simulator. 
The environment tracks BG history and the time since administration for all prior doses, and uses these quantities to construct the next observation. 
The reward function penalizes hypo- and hyperglycemia while incorporating dose size and BG rate-of-change.  
Resetting initializes the patient model, clears all histories, and returns the constructed feature vector for the first time step.

\paragraph{Concept Banks}
We give details on the concept banks used for each of our five environments: 
\begin{enumerate}
    \item \textbf{CartPole} originally has a state space consisting of four numbers: the position, velocity, angle, and angular velocity. We discretize each of these with two thresholds for position and angle, and four thresholds for angular velocity and velocity. This gives a total of 12 concepts. 
    \item \textbf{MiniGrid} is characterized through the agent's position, door location, key obtaining, and door unlock status. 
    For MiniGrid, we use a fully symbolic state representation consisting of $12$ discrete features extracted from the final observation frame. These features encode the agent's $(x,y)$ position, orientation, object locations, and local binary indicators. Concretely, the features correspond to: the agent's $x$-position and $y$-position; the agent's direction; the $x$- and $y$-positions of the key; the $x$- and $y$-positions of the door; a binary indicator for whether the door is open; and four binary indicators denoting the presence of obstacles to the right, left, down, and up of the agent. This totals to 44 concepts, noting that each feature is discretized into multiple binary concepts. 
    \item \textbf{Pong} has concepts derived from object-centric state variables computed from the last one or two frames of the observation stack. Base features include continuous-valued positions, velocities, and relative offsets of the agent paddle, ball, and opponent paddle.
    Specifically, we extract the agent paddle's vertical position; the ball's horizontal and vertical positions; the opponent paddle's vertical position; the ball's horizontal and vertical velocities; the opponent paddle's vertical velocity; and several relative position features, including paddle--ball and opponent--ball offsets. Position features are centered and normalized to lie approximately in $[-1,1]$, while velocity features are computed as frame-to-frame differences, clipped to a fixed range, and rescaled.
    This totals 228 concepts.
    \item \textbf{Boxing} concepts are constructed from the positions and movements of the player and opponent sprites. Base features include the normalized $x$- and $y$-positions of both the player and the enemy in the final observation frame, as well as their velocities computed as differences between the last two frames. We additionally include relative position features capturing the horizontal and vertical offsets between the player and the enemy.
    These concepts capture spatial dominance, motion direction, and relative positioning between the two agents.
    \item \textbf{Glucose} has a state which consists of six continuous-valued physiological or control-related variables extracted from the final observation frame. Each feature corresponds to a scalar quantity, such as an absolute level, rate of change, or control signal, and no temporal differencing is applied.
    Concepts are defined by thresholding each feature at a hand-specified set of values chosen to reflect the feature's numerical range and semantics. The threshold sets vary across features and include both fine-grained thresholds near zero and coarser thresholds for larger-magnitude variables. 
\end{enumerate}

\paragraph{Training Models}
We provide details on model training for each environment below. 
For all policies, we use Proximal Policy Optimization (PPO)~\cite{ppo}, with environment-specific hyperparameters. 
For CartPole, we train using $n_{\text{steps}}=1024$, a batch size of $256$, and $4$ epochs per update, with an entropy coefficient of $5\times10^{-4}$, a learning rate of $10^{-4}$, and a value-function loss coefficient of $0.5$. For MiniGrid, we again use $n_{\text{steps}}=1024$ but increase the batch size to $1024$, while keeping $4$ training epochs per update. For Pong, we use $n_{\text{steps}}=1024$, a batch size of $256$, and $5$ epochs per update, with a learning rate of $10^{-4}$. Finally, for Boxing, we use $n_{\text{steps}}=1024$, a batch size of $1024$, and $4$ epochs per update. All experiments are run on a GPU using CUDA.

Next, we give details on training interpretable policies, which again use PPO~\cite{ppo} but with smaller or structured network architectures to encourage interpretability. For CartPole, we use a two-layer policy network with hidden dimensions $[64,64]$, a rollout length of $256$, batch size $128$, and $4$ epochs per update, with a learning rate of $10^{-4}$, zero entropy regularization, a clipping range of $0.1$, and a target KL divergence of $0.01$. For MiniGrid, we use a learning rate of $3\times10^{-4}$ with $n_{\text{steps}}=1024$, batch size $1024$, and $4$ epochs per update. For Pong, we use a larger two-layer architecture with hidden dimensions $[256,256]$, together with $n_{\text{steps}}=1024$, batch size $256$, $5$ epochs per update, and a learning rate of $10^{-4}$. For Boxing, we use different hyperparameters depending on the experimental setting: in imperfect or intervention-based settings, we use a wider architecture $[512,256]$, larger rollouts ($n_{\text{steps}}=2048$), batch size $512$, $10$ epochs per update, a learning rate of $2\times10^{-4}$, entropy coefficient $0.01$, and clipping fraction $0.1$; otherwise, we use a smaller architecture $[128,128]$ with $n_{\text{steps}}=128$, batch size $256$, $4$ epochs per update, and entropy coefficient $0.01$. Finally, for the Glucose environment (including both processed and raw variants), we use $n_{\text{steps}}=1024$, batch size $256$, $5$ epochs per update, a learning rate of $3\times10^{-4}$, and an entropy coefficient of $0.01$, which we found necessary for stable training.

For training concept predictors, we use the Nature CNN backbone~\cite{nature_cnn}. 
We train all CNNs for 25 epochs using 50k rollout steps. 
All concept predictors are at least 95\% accurate per concept in MiniGrid, Pong, and Boxing. For CartPole, velocity-related concepts are at least 85\% accurate, while all other concepts are at least 95\% accurate.  
Concept predictions are kept as logits to preserve information rather than rounding. 
We train concept-based models under the sequential paradigm~\cite{cbms}, where we first train concept predictors $g_{\mathbf{c}}$, then learn interpretable policies $\pi^{*}_{\mathbf{c}}$. 

We train all models on a server running Ubuntu 20.04 (Linux kernel 5.15) with 4 NVIDIA RTX A6000 GPUs. 

\section{Experimental Details}
\label{sec:experimental_details}
We run all experiments for 6 seeds and report the standard error. 
We normalize rewards per-environment by setting the following min/max ranges for reward: CartPole is from 0 to 500, MiniGrid is from 0 to 1, Pong is from -21 to 21, Boxing is from 0 to 100, and Glucose is from -20 to 40. 
For the first four environments, the normalization comes from bounds on the reward, while for Glucose, we develop this normalization factor based on the performance of all policies. 
For intervention experiments, we weaken training for all but the CartPole environment, so concept predictors are trained for only 1 epoch, and policies are trained for 250k, 4M, and 2M epochs for MiniGrid, Pong, and Boxing, respectively. 
We assume that intervening on concept $j$ results in perfect accuracy for the concept, $\delta_{j} = 1$. 
For all policies, we use Proximal Policy Optimization (PPO)~\cite{ppo}, with environment-specific hyperparameters. 

\section{Algorithm Implementation Details}
\label{sec:algorithm_details}
We estimate Q-values using a stable TD learning procedure~\cite{rl_textbook}. A neural Q-network with two ReLU hidden layers and orthogonal initialization is trained together 
with a slowly updated target network. At each step, the agent stores 
transitions $(s, a,r,s')$ into a replay buffer and updates the parameters by 
minimizing the Double--DQN TD loss.
\[
\ell \;=\; 
\Bigl(Q_\theta(s,a) 
 - \bigl[r + \gamma\, Q_{\theta^-}\bigl(s',\,\arg\max_{a'} Q_\theta(s',a')\bigr)\bigr]\Bigr)^2 .
\]
Gradients are clipped for stability, and the target network is synchronized 
periodically. After training, the Q-value $Q(s, a)$ for any state--action pair 
is obtained by a single forward pass through the learned network.

The greedy method measures the within-class standard deviation after splitting on each concept then sorts by this. 
That is, for each concept, we sum the standard deviation for Q values when the concept is on and off, summed across all actions. 
The intuition is that if a concept is highly relevant to the reward or dynamics, states where the concept is on should have similar Q-values, and states where it is off should have similar Q-values.

For the DRS and DRS-log algorithms, we fix $\rho=0.75$; if no solution is found, we decrement $\rho$ by $0.05$ until a solution is found. 
When using the DRS algorithm in practice, we relax \emph{P1c} by dropping the constraint, as this improves performance in practice (see Appendix~\ref{sec:other_algorithms}). 
We sample $D_{s,s'}$ from rollouts of $\pi^{*}$ with 200k timesteps, while we compute $\mathbf{1}[\pi^{*}(s) \neq \pi^{*}(s')]$ with 10k timesteps from the rollout of $\pi^{*}$. 

\section{Hyperparameter Selection for DRS}
\label{sec:other_algorithms}
We compare our DRS algorithm against two variations of the algorithm: one which incorporates \emph{P1c}, along with others which vary the value of $\rho \in \{0,0.5,0.75,0.99\}$. 
We compare these in both the perfect and imperfect MiniGrid and CartPole algorithms as a simple evaluation to understand performance. 

\begin{figure}
    \includegraphics[width=0.45\textwidth]{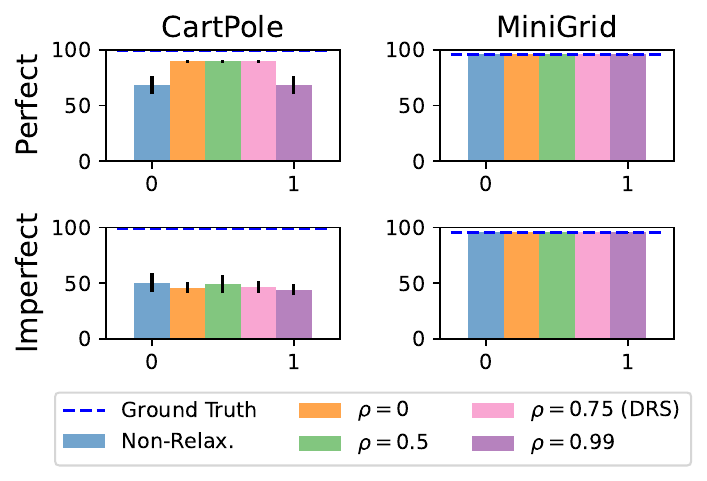}
    \caption{We assess the performance of a) not relaxing constraint \emph{P1c} for DRS and b) varying $\rho$ for CartPole and MiniGrid. We find that $\rho=0.75$ performs well across perfect and imperfect settings, while introducing \emph{P1c} leads to lower performance in CartPole. }
    \label{fig:other_methods}
\end{figure}

In Figure~\ref{fig:other_methods}, we find that re-introducing \emph{P1c} leads to worse performance in perfect CartPole, while setting $\rho=0.99$ or $\rho=0.0$ also leads to worse performance.
Here, enforcing \emph{P1c} forces an ordering over state pairs to cover, which can be too strict in practice, leading us to relax it when constructing DRS. 
Similarly, setting $\rho=0.99$ can be too strict, as it forces DRS to separate 99\% of state pairs with different optimal actions. 

\section{Varied Number of Concepts Selected}
\label{sec:varied_k}
\begin{figure*}[h]
    \centering 
    \includegraphics[width=\textwidth]{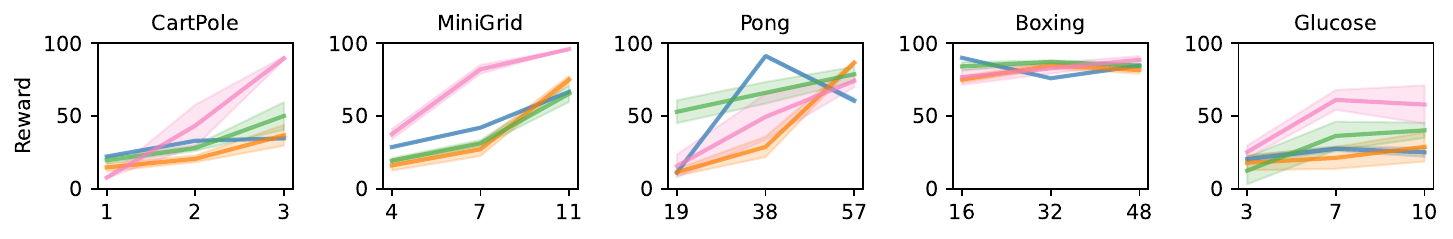}
    \includegraphics[width=\textwidth]{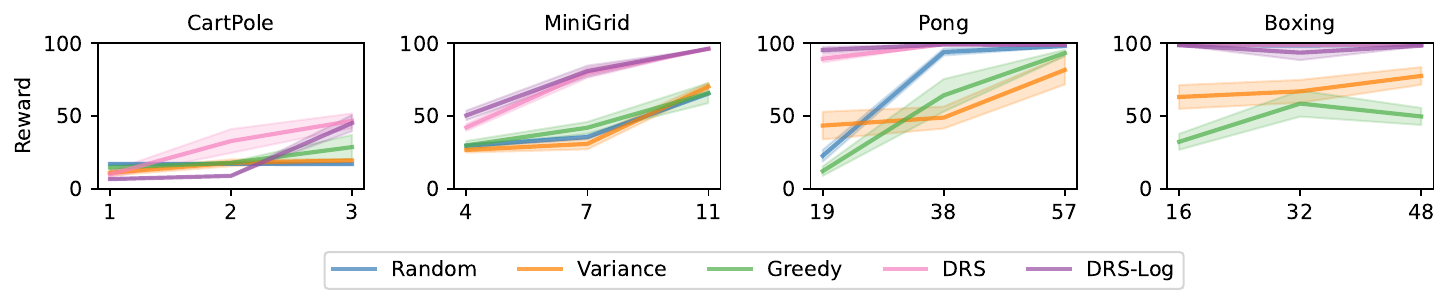}
    \caption{We compare concept selection algorithms while varying the number of concepts selected $k$ with perfect (top) and imperfect (bottom) concept predictors. We find that for CartPole and MiniGrid, DRS outperforms baselines across values of $k$, while for environments such as Boxing and Glucose, a threshold number of concepts is needed before DRS outperforms baselines. }
    \label{fig:varied_k_comparison}
\end{figure*}

We vary $k \in \{\frac{K}{12},\frac{K}{6},\frac{K}{4}\}$ in Figure~\ref{fig:varied_k_comparison} for both perfect and imperfect settings. 
We find that better concept selection algorithms can improve performance across many values of $k$. 
For example, in the perfect environment setting for CartPole, we find that the DRS algorithm improves performance even with only $k=2$ concepts, leading to an 32\% improvement over the next best alternative. 
Similarly, such a trend is also present in MiniGrid, where the DRS algorithm improves performance by 96\% when $k=7$. 
In some environments, many concepts are necessary to achieve high performance; for example, in Glucose, while the DRS algorithm is 24\% better than alternatives when $k=3$, it achieves 131\% better than all alternatives once $k=10$. 

\section{Qualitative Analysis}
\label{sec:qualitative}
\begin{table}[t]
\centering
\small
\caption{Average number of concepts selected per group by DRS in MiniGrid, with total concepts per group.}
\label{tab:minigrid_concepts}
\begin{tabular}{lcc}
\toprule
\textbf{Concept Group} & \textbf{Avg. Selected} & \textbf{Total} \\
\midrule
\multicolumn{3}{l}{\textit{Agent}} \\
Position X     & 1 & 5 \\
Position Y     & 2 & 5 \\
Direction      & 3 & 4 \\
Action: Right  & 0 & 2 \\
Action: Left   & 1 & 2 \\
Action: Down   & 0 & 2 \\
Action: Up     & 1 & 2 \\
\midrule
\multicolumn{3}{l}{\textit{Key}} \\
Position X     & 0 & 5 \\
Position Y     & 1 & 5 \\
\midrule
\multicolumn{3}{l}{\textit{Door}} \\
Position X     & 0 & 5 \\
Position Y     & 1 & 5 \\
Open State     & 1 & 2 \\
\bottomrule
\end{tabular}
\end{table}

We next provide intuition for why DRS improves performance by comparing the set of concepts selected. 
In MiniGrid, concepts include agent location, door location, and agent direction  (see Table~\ref{tab:minigrid_concepts}).
DRS consistently selects direction-related concepts, which are not selected by variance or greedy. 
This occurs because the direction-related concepts have high action variability, but are also decision-relevant. 
Knowing which direction an agent faces is critical to train effective policies.

\section{Runtime of Algorithms}
\label{sec:runtime}
\begin{figure*}[h]
    \centering 
    \includegraphics[width=\textwidth]{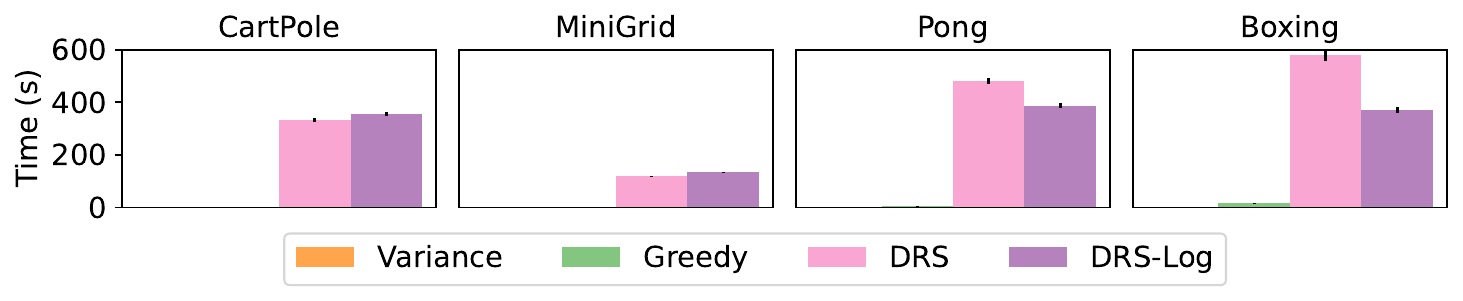}
    \caption{We compare the time to compute across methods and datasets, and find that all algorithms run in under ten minutes. While DRS and DRS-log take longer than variance and greedy, they still run quickly across environments. }
    \label{fig:timing}
\end{figure*}

We compare the runtime of the Variance, Greedy, DRS and DRS-log algorithms across different environments (while we note that the random selection algorithm essentially takes no time).
In Figure~\ref{fig:timing}, we show that all algorithms run in under ten minutes across all methods.
This runtime is much smaller than the time needed to train the policies themselves, which can take hours. 
While more complex environments such as Pong or Boxing take longer for the DRS and DRS-log algorithms, they still run quickly. 
Much of this time is spent solving the optimization problem and for running rollouts to compute optimal actions. 
Moreover, we note that $n_d$, or the effective dimension, is small in practice. 
For example, in our rollouts for Pong, we find that $n_d = 1951$ when rolled out for 200,000 steps, which is much less than $2^{k} = 2^{228}.$
Additionally, even when we double the number of concepts in Pong (using redundant concepts), we find that all algorithms run in under 10 minutes (with only a 23\% increase in the time taken) even with 400+ candidate concepts. 
We note that we precomputed $Q$ values for all methods, though the computation of $Q$ values takes under ten minutes as well. 

\section{Training Curves for CartPole}
\label{sec:num_concepts_cartpole}
\begin{figure}[t]
    \centering 
    \includegraphics[width=0.5\textwidth]{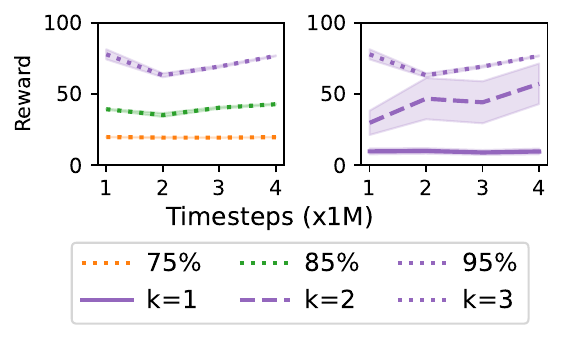}
    \caption{We vary the number of timesteps that we train policies for in CartPole, while also varying the accuracy of concept predictors (left) or the number of concepts selected (right). When the number of concepts is small or the accuracy of the concepts is low, policies fail to train, with the efficiency of training increasing with both parameters.}
    \label{fig:num_concepts_cart_pole}
\end{figure}

In Figure~\ref{fig:num_concepts_cart_pole}, we plot training curves when varying the concept predictor accuracy (left) and number of concepts (right). 
Our results are noisier compared to those for MiniGrid, yet in both, we find that when concept predictor accuracy is low or the number of concepts is low, policies fail to train. 
Concepts train slightly faster when going from 75\% to 85\% accurate concept predictors, and similarly the maximum performance increases when increasing the number of concepts selected.
However, both of these trends are noisy in CartPole due to the large amount of stochasticity; even small differences in the set of concepts selected or the training dynamics can greatly impact policy performance. 

\section{Varying the Intervention Percentage}
\label{sec:intervention_percentage}
\begin{figure}[h]
    \centering 
    \includegraphics[width=0.4\textwidth]{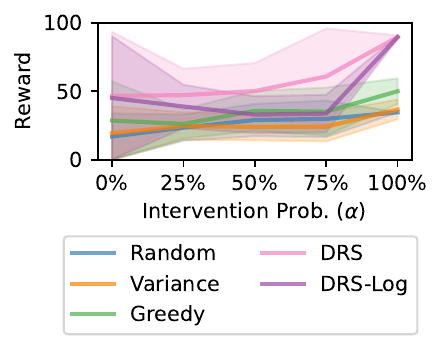}
    \caption{We assess test-time intervention performance when varying the level of intervention $\alpha$ in CartPole. Across all $\alpha$, the DRS algorithm performs best, with steadily improving performance as $\alpha$ is increased.}
    \label{fig:intervention_percent}
\end{figure}

In Figure~\ref{fig:intervention_percent}, we vary $\alpha$ and analyze the reward in CartPole.
We find that the DRS algorithm improves performance across values of $\alpha$, and steadily improves as $\alpha$ increases.
When $\alpha=0.5$, DRS achieves 40\% better than baselines and at $\alpha=0.75$, DRS achieves a 73\% better performance.  
DRS also improves as $\alpha$ increases; at $\alpha=0.5$ DRS improves performance compared to $\alpha=0$ by 7\%, and at $\alpha=0.75$, DRS improves by 30\%.

\section{Concept Selection Algorithms in Supervised Settings}
\label{sec:supervised_algorithms}
We extend the algorithms from Section~\ref{sec:imperfect} in supervised settings to use with the CUB dataset. 
In this setting, rollouts are replaced by labeled examples, and concept selection is performed using empirical statistics computed over the training set. 
The variance method selects concepts whose presence most evenly partitions the data, favoring concepts with high empirical variability. 
The greedy method ranks concepts by their association with the target label, selecting those with the strongest marginal predictive signal.
The DRS method formulates concept selection as a weighted hitting-set problem over pairs of training examples with differing labels, selecting a fixed-size subset of concepts that maximizes deterministic coverage of label-separating pairs. 
Here, weights capture whether $s$ and $s'$ have the same label; $D_{s,s'} = 1$ if $s$ and $s'$ have different labels, essentially capturing a one-step difference in Q values. 
Essentially, this extends coverage over $Q$ values to the supervised setting. 
Finally, DRS-log extends this formulation by incorporating concept accuracies into the coverage constraints, replacing hard coverage with a probabilistic notion of separation that accounts for uncertainty in concept predictions while retaining the same budgeted optimization structure.

\section{Selecting Concepts with Imperfect Policies}
\label{sec:groundtruth_model}
Throughout our work, we have assumed access to a perfect model $\pi^{*}$. 
In this section, we relax that assumption by training optimal policies $\pi^{*}$ in Minigrid for varied numbers of timesteps. 
We compare the performance of two baselines (variance and greedy) and our two algorithms (DRS and DRS-log) when varying the number of gold timesteps. 
We do not test the random selection algorithm because it is independent of $\pi^{*}$. 

\begin{figure}
    \centering 
    \includegraphics[width=0.45\textwidth]{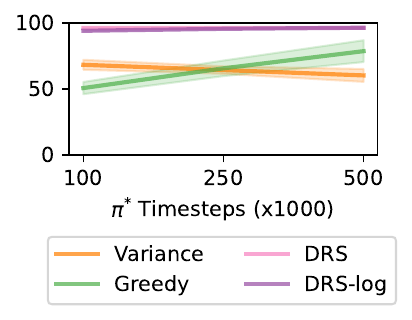}
    \caption{We vary the number of timesteps we use to train the policy $\pi$ without concepts in MiniGrid. Typically, we have $\pi = \pi^{*}$, but here, we demonstrate that the DRS and DRS-log algorithms perform well even when $\pi$ is trained for only 100k timesteps.} 
    \label{fig:varied_gold}
\end{figure}
In Figure~\ref{fig:varied_gold}, we find that the DRS and DRS-log algorithms perform well across the number of timesteps $\pi^{*}$ is trained for. 
Both algorithms achieve optimal performance regardless of the number of timesteps $\pi^{*}$ is trained for.
The greedy and variance algorithms perform at least 18\% worse no matter how many timesteps $\pi^{*}$ is trained for. 
Therefore, even with access to an approximate policy $\hat{\pi}$, the DRS and DRS-log algorithms perform well. 

\section{Proofs}
\label{sec:proofs}
\begin{theorem}
    Optimizing the objective in Equation~\ref{eq:objective} is NP-hard, even under the following restrictions: (i) concept predictors are perfect, and (ii) the optimal action-value function $Q^*$ is known.
\end{theorem}

\begin{proof}
We prove the claim via a polynomial-time reduction from the \emph{Weighted Maximum Coverage} problem.

An instance of weighted maximum coverage consists of a universe
$U=\{e_1,\dots,e_n\}$ with weights $w_i\ge 0$, a collection of sets
$\mathcal S=\{S_1,\dots,S_m\}$, and a budget $k$. The goal is to select at most
$k$ sets maximizing $\sum_{e_i\in \cup_{S_j\in\mathbf c} S_j} w_i$.

Given such an instance, we construct a deterministic episodic MDP.
For each element $e_i$, introduce two terminal states $s_i^L$ and $s_i^R$ and an
initial state $s_0$. From $s_0$, the MDP transitions uniformly to one of the
states $\{s_i^L,s_i^R\}_{i=1}^n$. The action space is
$\mathcal A=\{a_L,a_R\}$, and rewards are defined by
$R(s_i^L,a_L)=w_i,\quad R(s_i^L,a_R)=0,\qquad
R(s_i^R,a_R)=w_i,\quad R(s_i^R,a_L)=0.$
Thus the optimal action-value function $Q^*$ is known.

We define a concept bank as follows. For each set $S_j\in\mathcal S$, introduce a
binary concept $c_j$ such that
\[
c_j(s_i^L)=0,\quad c_j(s_i^R)=1 \;\;\text{if } e_i\in S_j,
\]
and $c_j(s_i^L)=c_j(s_i^R)=0$ otherwise.

Fix any subset of concepts $\mathbf c$ with $|\mathbf c|\le k$, and let
$\pi^{*}_{\mathbf{c}}$ be the optimal policy that conditions only on
$g_{\mathbf c}$. If no selected concept distinguishes $(s_i^L,s_i^R)$, then
$\pi^{*}_{\mathbf{c}}$ must take the same action in both states, yielding expected
reward $w_i/2$. If the pair is distinguished by at least one selected concept,
the policy can take the optimal action in each state and obtain reward $w_i$.

Let $C(\mathbf c)\subseteq U$ be the elements whose corresponding state pairs are
distinguished by $\mathbf c$. The expected return of $\pi^{*}_{\mathbf{c}}$ is
\[
\sum_{e_i\in C(\mathbf c)} w_i + \sum_{e_i\notin C(\mathbf c)} \frac{w_i}{2}
= \frac{1}{2}\sum_i w_i + \frac{1}{2}\sum_{e_i\in C(\mathbf c)} w_i .
\]
Since the first term is constant, maximizing Equation~\ref{eq:objective} over
$|\mathbf c|\le k$ is equivalent to maximizing
$\sum_{e_i\in\cup_{S_j\in\mathbf c} S_j} w_i$, i.e., solving weighted maximum
coverage.

Therefore, a polynomial-time algorithm for optimizing
Equation~\ref{eq:objective} would imply a polynomial-time algorithm for weighted
maximum coverage. Since weighted maximum coverage is NP-hard, the concept
selection problem is NP-hard as well.
\end{proof}

\propconceptabstraction*
\begin{proof}
By assumption, whenever $g_{\mathbf{c}}(s) = g_{\mathbf{c}}(s')$ we have
$\max_a |Q(s,a)-Q(s',a)| \le \epsilon$.
Thus, $g_{\mathbf{c}}$ induces an $\epsilon$-approximate state abstraction
in the sense of Theorem~\ref{thm:approximate_state_abstraction}.
The claimed bound on the value loss of the induced policy $\pi^{*}_{\mathbf{c}}$
follows directly from that theorem.
\end{proof}

\theoremhuman*
\begin{proof}
By Assumption~\ref{asm:abstraction_dominance}, a smaller abstraction error implies a higher value under perfect concept accuracy.
Therefore, under perfect intervention, 
$\epsilon(g_{\mathbf c}) \le \epsilon(g_{\mathbf c'})$.

Human test-time intervention improves concept accuracy coordinate-wise, so for any
$\alpha$, the resulting accuracy vector $\delta^{(\alpha)}$ satisfies
$\delta^{(\alpha)} \preceq \mathbf 1$.
By Assumption~\ref{asm:monotone_dominance}, relative performance rankings between policies induced by
$\mathbf c$ and $\mathbf c'$ are preserved as accuracy improves.
Therefore, the ordering induced at perfect accuracy extends to all intermediate
levels of intervention, completing the proof.
\end{proof}

\theoremperfect*
\begin{proof}
By constraint~\textup{(P1a)}, the abstraction $g_{\mathbf c}$ uses at most $k$
concepts.

By construction, $Y_{s,s'}^* = 0$ if and only if $g_{\mathbf c}(s) = g_{\mathbf c}(s')$ or $Y_{t,t'}^* = 0$ where $D_{t,t'} \geq D_{s,s'}$.
Let the state pairs be indexed so that
$D_1 \ge D_2 \ge \cdots$, and note that constraint~\textup{(P1c)} enforces
$Y_1^* \ge Y_2^* \ge \cdots$.
Since $Y_{s,s'}^* \in \{0,1\}$, there exists a threshold $\tau$ such that
$Y_{s,s'}^* = 0$ for all pairs with $D_{s,s'} \le \tau$, and
$Y_{s,s'}^* = 1$ for all pairs with $D_{s,s'} > \tau$.

Therefore,
\[
\epsilon(g_{\mathbf c})
=
\max_{s,s' : Y_{s,s'}^* = 0} D_{s,s'}
=
\tau.
\]

Now consider any abstraction $g$ using at most $k$ concepts. For $g$ to achieve
abstraction error strictly less than $\tau$, it must separate all state pairs
with $D_{s,s'} \ge \tau$, which would contradict the optimality of
$(\mathbf x^*, \mathbf y^*)$ for \emph{(P1)}. Hence, no such abstraction exists,
and $g_{\mathbf c}$ is minimax optimal.
\end{proof}

\theoremrobust* 
\begin{proof}
Since $\|\hat Q - Q^{\pi^*}\|_\infty \le \delta$, for all $s,s'$ we have
\begin{align} 
|\hat D_{s,s'} - D_{s,s'}|
\\ =
\left| \max_a |\hat Q(s,a)-\hat Q(s',a)|
-
\max_a |Q(s,a)-Q(s',a)| \right|
\\ \le 2\delta.
\end{align}
Here, this holds because of the Lipschitzness of the max operator and triangle inequality. 

Fix any abstraction $g$. Then
\[
\left|
\max_{s,s' : g(s)=g(s')} \hat D_{s,s'}
-
\max_{s,s' : g(s)=g(s')} D_{s,s'}
\right|
\le 2\delta.
\]
In particular,
\[
\epsilon_{\hat Q}(g) \le \epsilon_Q(g) + 2\delta.
\]

Let $g_{\mathbf c^*}$ be the abstraction minimizing $\epsilon_Q(g)$, and
$g_{\hat{\mathbf c}}$ be the abstraction minimizing $\epsilon_{\hat Q}(g)$.
By optimality of $g_{\hat{\mathbf c}}$ under $\hat d$,
\[
\epsilon_{\hat Q}(g_{\hat{\mathbf c}})
\le
\epsilon_{\hat Q}(g_{\mathbf c^*})
\le
\epsilon_Q(g_{\mathbf c^*}) + 2\delta.
\]
Applying the bound in the opposite direction yields
\[
\epsilon_Q(g_{\hat{\mathbf c}})
\le
\epsilon_{\hat Q}(g_{\hat{\mathbf c}}) + 2\delta
\le
(\epsilon^* + 2\delta) + 2\delta = \epsilon^* + 4\delta.
\]

The value-function bound then follows directly from
Proposition~\ref{thm:approximate_concept_abstraction}.
\end{proof}

\lemmaperturbation*
\begin{proof}
Since $\delta_i < 1$ for all $i$, we may choose $T_i = \{s\}$ for any state
$s \in \mathcal S$.
Let $(s^*,s'^*) \in \arg\max_{s,s'} D_{s,s'}$, and define
\[
P = \{ i : c_i(s^*) \neq c_i(s'^*) \}.
\]
For each $i \in P$, let $T_i = \{s^*\}$ and apply $f^{(T_i)}$, which flips the value
of $c_i$ at $s^*$; for $i \notin P$, let $f^{(T_i)}$ be the identity.
After applying all perturbations, we have
\[
f(g_{\mathbf c})(s^*) = f(g_{\mathbf c})(s'^*).
\]
Since $(s^*,s'^*)$ maximizes $D_{s,s'}$, it follows that
\[
\epsilon(f(g_{\mathbf c})) = \max_{s,s'} D_{s,s'}.
\]
\end{proof}

\onecolumn

\end{document}